\documentclass[]{ailabs}

\usepackage{enumitem}

\usepackage{multirow}
\usepackage{graphicx}

\usepackage{amsthm}
\usepackage{wrapfig}

\newtheorem{theorem}{Theorem}
\newtheorem{definition}{Definition}

\title{Be Wary of Your Time Series Preprocessing}

\author[1]{Sofiane Ennadir}
\author[2,*]{Tianze Wang}
\author[1]{Oleg Smirnov}
\author[1]{Sahar Asadi}
\author[1]{Lele Cao}

\affiliation[1]{King AI Labs, Microsoft Gaming}
\affiliation[2]{Kreditz AB}

\contribution[*]{Work done while at AI Labs}

\correspondence{\email{sofiane.ennadir@king.com}}
\date{February 2025}

\abstract{%
Normalization and scaling are fundamental preprocessing steps in time series modeling, yet their role in Transformer-based models remains underexplored from a theoretical perspective. In this work, we present the first formal analysis of how different normalization strategies, specifically instance-based and global scaling, impact the expressivity of Transformer-based architectures for time series representation learning. We propose a novel expressivity framework tailored to time series, which quantifies a model’s ability to distinguish between similar and dissimilar inputs in the representation space. Using this framework, we derive theoretical bounds for two widely used normalization methods: Standard and Min-Max scaling. Our analysis reveals that the choice of normalization strategy can significantly influence the model's representational capacity, depending on the task and data characteristics. We complement our theory with empirical validation on classification and forecasting benchmarks using multiple Transformer-based models. Our results show that no single normalization method consistently outperforms others, and in some cases, omitting normalization entirely leads to superior performance. These findings highlight the critical role of preprocessing in time series learning and motivate the need for more principled normalization strategies tailored to specific tasks and datasets.}

\begin{document}

\maketitle

\section{Introduction}
Deep learning models have achieved significant success across various domains, notably in natural language processing and computer vision, where Transformer-based architectures~\citep{vaswani2017attention} have become the \textit{de facto} standard for many tasks. Motivated by these advances, the time series domain has also seen increasing adoption of Transformer models~\citep{wang2025deeptimeseriesmodels}, yielding strong performance on downstream tasks such as forecasting, classification, and anomaly detection. Similar to developments in other domains, recent work has introduced foundation models for time series~\citep{Goswami2024moment, Das2025timesfm}, capable of generalizing across tasks with minimal adaptation. These models are typically pre-trained on auxiliary objectives, including next-token prediction, masked reconstruction~\citep{zhang2022survey}, or recently joint-embedding learning~\citep{ennadir2024joint}, and can be adapted via head-only fine-tuning, full fine-tuning, or parameter-efficient methods (PEFT)~\citep{han2024parameter} such as Low Rank Adaptation (LoRA)~\citep{hu2022lora}.

Despite the growing body of research and work on architecture adaptations \cite{zhou2021informer, wu2021autoformer, nie2023patchtst}, pre-training strategies \cite{zhang2022self}, and tokenization schemes for time series, one critical component remains underexplored: the role of input scaling and transformation prior to model ingestion. Typically, in most Transformer-based approaches for time series, the input is scaled before computing attention. The main idea is that this scaling helps in the stability of the training by controlling the gradient for instance. However, and to our knowledge, the impact of this scaling step on model behavior and performance in the downstream tasks has not been rigorously analyzed. Prior work has largely treated input normalization as a preprocessing detail, rather than a design choice with theoretical and empirical implications.

In this perspective, this early work aims to reduce this gap by theoretically analyzing the effect of input transformation part. Specifically, we aim to understand how the chosen scaling distribution and strategy affects the overall performance of the model. In this direction, we start by formalizing the concept of expressivity of a time series Transformer-based model and consequently use this formalization to theoretically understand how some scaling methods that are widely used in practice affect the overall model. We specifically derive some insights on how to choose the right strategy depending on the dataset and also on the considered downstream task. Then we empirically analyze this effect and showcase that actually the choice of scaling method can result in very diverse results. Therefore, choosing the right strategy is critical. While we do not specifically propose a new scaling or preprocessing method for time series input data, our work lays down an important theoretical analysis to understand this effect better, and it opens a new research direction that could tackle this problem. We summarize our contributions as follows:
\begin{itemize}
    \item We present a theoretical framework for analyzing how input scaling affects the expressivity of Transformer-based models for time series.
    \item We provide empirical evidence validating our theoretical insights across multiple datasets, tasks, and model architectures.
\end{itemize}

\section{Related Work}\label{sec:related_work}
The widespread use of time series in real-world applications has driven the development and adaptation of deep learning models to effectively learn temporal representations. Early approaches were primarily based on recurrent neural networks (RNNs) in its different versions \citep{}, which naturally model temporal dependencies via their sequential structure. Building on this, LSTNet~\citep{lai2018modeling} introduced a hybrid architecture that integrates convolutional neural networks (CNNs) with recurrent skip connections, allowing the model to capture both short-term and long-term temporal patterns.

\noindent With the success of Transformer-based models in other domains, their adaptation to time series has gained significant attention. Informer~\citep{zhou2021informer} proposes a ProbSparse self-attention mechanism coupled with distillation techniques to improve efficiency by focusing on the most informative keys. Autoformer~\citep{wu2021autoformer} incorporates ideas from traditional time series analysis, such as series decomposition and autocorrelation, into the Transformer architecture. More recently, PatchTST~\cite{nie2023patchtst} addresses limitations in prior designs by segmenting time series into patches and processing each channel independently. This approach preserves semantic structure while improving both model efficiency and predictive performance.

\noindent In addition to supervised models, self-supervised and foundation models have become increasingly prominent in time series representation learning. TimesFM~\cite{Das2025timesfm} presents a general-purpose forecasting model trained on a large corpus of public and proprietary datasets. Leveraging a Transformer-based encoder-decoder architecture, it demonstrates strong zero-shot performance across diverse applications, including forecasting, classification, and anomaly detection — particularly in low-resource settings. On a broader scope, MOMENT~\cite{Goswami2024moment} introduces a universal foundation model that unifies forecasting with representation learning across multiple data modalities. Through large-scale pretraining and adaptive fine-tuning, MOMENT supports a wide range of tasks, including classification, imputation, and anomaly detection. Its architecture is explicitly designed to accommodate multimodal inputs, such as categorical attributes and textual context, enabling flexible and context-aware modeling.

\section{Preliminaries}\label{sec:prelems}

We begin by introducing key definitions and notations that form the foundation of our work and specifically our theoretical analysis.

\noindent \textbf{Time Series Modeling.} We consider the general task of representation learning for time series data, with a focus on the multivariate setting. Formally, let $\mathcal{X} = \{x_1, x_2, \ldots, x_N\}$ denote a collection of $N$ time series instances, where each $x_j$ is of length $\ell$ and consists of $k$ channels. The goal is to learn a function $f: \mathcal{X} \rightarrow \mathcal{R}$ that maps each input time series to a representation $h \in \mathcal{R} \subset \mathbb{R}^m$, where $m$ is the chosen dimension of the learned hidden space.
This learned representation can serve different purposes depending on the downstream task. For instance in supervised settings, where a corresponding label set $\mathcal{Y} = \{y_1, \ldots, y_N\}$ is available, the representations may be passed to a prediction head (e.g., a linear classifier) for tasks such as classification, regression, or forecasting. In self-supervised scenarios, the function $f$ is instead trained using auxiliary objectives, with the goal of capturing generalizable patterns that transfer to various downstream tasks.

\noindent \textbf{Transformer-Based Models.} While various architectures have been proposed for time series modeling (as denoted in Section \ref{sec:related_work}), this work focuses on the Transformer architecture~\citep{vaswani2017attention}, which serves as the backbone of most recent foundation models. Resulting from its flexibility and strong empirical performance, the Transformer architecture has become a widely adopted model across a range of time series tasks, including forecasting, classification, and anomaly detection. Numerous adaptations have further extended its applicability to temporal data, making it a reference point for state-of-the-art performance.

Given a multivariate time series input, the data is typically segmented into tokens before being processed by the Transformer. Let $X \in \mathcal{X} \subseteq \mathbb{R}^{n \times d}$ denote the resulting sequence of $n$ tokens, where each token $x_i \in \mathbb{R}^d$. The core of the Transformer architecture is the \textit{self-attention} mechanism~\citep{vaswani2017attention}, which computes contextualized representations by attending over all input tokens. Specifically, for learnable \textit{query}, \textit{key}, and \textit{value} projection matrices $W^Q, W^K, W^V \in \mathbb{R}^{d \times (d/H)}$, the output of a single attention head $\mathrm{AH}$ is computed as:
\begin{equation*}\label{eq:scaled_dot_product_attention}
\mathrm{AH}(X) = \operatorname{softmax}\!\left(\frac{(XW^Q)(XW^K)^\top}{\sqrt{d/H}}\right)(XW^V),
\end{equation*}
\noindent where $H$ denotes the number of attention heads, and $d/H$ is the dimensionality per head. In practice, $H$ attention heads are computed in parallel and then concatenated and projected via a learnable output matrix $W^O \in \mathbb{R}^{d \times d}$ to obtain the multi-head attention (MHA) output:
\begin{equation*}\label{eq:multi_head_attention}
\mathrm{MH}(X) = \mathrm{concat}\bigl(\mathrm{AH}_1(X), \mathrm{AH}_2(X), \dots, \mathrm{AH}_H(X)\bigr) W^O.
\end{equation*}
Each Transformer block also includes a residual connection, layer normalization~\citep{lei2016layer}, and a position-wise feed-forward network (FFN). The FFN is composed of fully connected layers applied independently to each position, where each layer performs an affine transformation followed by a non linear activation function:
$$h^{(\ell)} = \sigma(W^{(\ell)} h^{(\ell-1)} + b^{(\ell)}), \text{ with } h^{(0)} = X.$$

\noindent \textbf{Time Series Preprocessing.} Given a multivariate time series input that has been tokenized into $X \in \mathcal{X} \subseteq \mathbb{R}^{n \times d}$, it is standard practice to apply preprocessing to each channel prior to model ingestion. Such preprocessing steps are crucial for improving training stability and overall model performance. We consider a general formulation in which each input channel is transformed by a distinct preprocessing function. Specifically, let $h_1, \ldots, h_d$ denote a set of functions, where each $h_i$ operates on the $i$-th channel:
\begin{align*}
    h_i: \quad & \mathcal{X}_i \subseteq \mathbb{R}^{\ell} \rightarrow \mathcal{H} \subseteq \mathbb{R}^{\ell} \\
    & X_i \mapsto h_i(X_i)
\end{align*}

\noindent where $X_i$ denotes the time series corresponding to channel $i$ with length $\ell$. This formulation allows for channel-specific preprocessing strategies, though in practice, another common approach used in practice is to apply the same transformation or use the same parameters across all channels.

\noindent A variety of preprocessing techniques are being used in practice. In our work, we focus first one the \textit{Standard scaling}, which consists of normalizing through the mean and standard deviation and the \textit{Min-Max scaling}, which rescales each channel to the $[0, 1]$ interval based on its minimum and maximum values. Such transformations standardize the input scale, aiding the optimization process during training.

\section{Expressivity of a Time Series Transformer}

In this section, we provide a theoretical analysis of the impact of commonly used input scaling methods in Transformer-based models (TBMs) for time series. We start by introducing a notion of expressivity tailored to the time series setting. Then, we introduce our considered problem setup. Finally, we use the provided framework to investigate how two widely adopted pre-scaling strategies influence the resulting model's expressivity.

\subsection{Expressivity of Transformer-Based Models}

As outlined in Section \ref{sec:prelems}, time series representation learning aims to map a time series to a fixed-dimensional vector $h \in \mathbb{R}^m$ that captures the salient characteristics needed for downstream tasks. Ideally, this representation should preserve meaningful similarities and differences between time series: similar inputs should yield similar representations, while dissimilar inputs should result in distinguishable outputs. This property is particularly important in settings such as classification, where a model must be able to easily differentiate and assign distinct labels to similar and dissimilar inputs from different classes.

To formally capture this behavior, we build on the framework proposed in prior work~\cite{ennadir2025poolwiselyeffectpooling}, and define a local neighborhood around a time series $X$ using a “similarity budget” $\epsilon$. Specifically, we define the $\epsilon$-similar neighborhood as the set:
$$ \mathcal{N}_{\epsilon}(X) = \left\{ \tilde{X} \in \mathcal{X} : \lVert X - \tilde{X} \rVert \leq \epsilon \right\},$$
where $\lVert \cdot \rVert$ denotes a norm-based distance defined on the input space $\mathcal{X}$. This neighborhood captures time series that are considered semantically similar to $X$. We then define the expressivity of a model $f$ as its ability to differentiate between such similar inputs in the learned representation space. Intuitively, if a model maps nearby time series to significantly different outputs, it has high expressivity in that region of the input space. Formally, we define the expressivity of a model $f$ under a distance threshold $\sigma$ as:
\begin{align*}\label{eq:experessivity}
\mathcal{E}_{\epsilon}[f] = \mathop{\mathbb{P}}_{X \sim \mathcal{D}_{\mathcal{X}}} \Bigl[ \tilde{X} \in \mathcal{N}_{\epsilon}(X) : d_{\mathcal{Y}}(f(\tilde{X}), f(X)) > \sigma \Bigr],
\end{align*}
where $\mathcal{D}_{\mathcal{X}}$ is the data distribution over the input space $\mathcal{X}$, and $d_{\mathcal{Y}}$ is a distance function in the output space $\mathcal{Y}$. In this work, we consider the $\ell_2$ norm for $d_{\mathcal{Y}}$, although other norms shall yield same insights given the equivalence of norm. Hence, choosing $\ell_1$ may be preferable in scenarios where outliers and sparse deviations are important to track. 

The quantity $\mathcal{E}_{\epsilon}[f]$ measures the probability that the model significantly separates two similar inputs by more than a threshold $\sigma$. A low value of $\mathcal{E}_{\epsilon}[f]$ implies the model is smooth in that neighborhood, preserving similarity. Conversely, a higher value suggests sharper distinctions in the learned representations. In  this perspective, we can introduce a formal definition of expressivity. 

\begin{definition}\label{def:expressivity}
Let $f \colon \mathcal{X} \subseteq \mathbb{R}^{n \times d} \rightarrow \mathcal{Y} \subseteq \mathbb{R}^d$ be a time series representation model. We say that $f$ is $(\epsilon, \sigma, \gamma)$\textit{-expressive} if $\mathcal{E}_{\epsilon}[f] \leq \gamma$.
\end{definition}

Note that this definition is general and can be applied to any representation model for time series. In this work, however, and as previously specified, we focus on the subclass of Transformer-based models, which currently dominate the literature due to their versatility and strong empirical results. Our goal is to investigate how different preprocessing strategies affect the expressivity of such models.

\paragraph{Problem Setup.} For our theoretical analysis, we consider a simplified setting where $f$ is a single-layer Transformer model based on self-attention with $H$ attention heads. To maintain generality while enabling tractable analysis, we assume that the non-linear activation functions used in the architecture are $1$-Lipschitz continuous~\citep{virmaux18}, a condition satisfied by many used and standard function in the literature.

\subsection{On the Effect of Normalization}
Normalizing the input time series has become a standard preprocessing step in most Transformer-based time series models. This operation is widely believed to improve training stability by mitigating issues such as vanishing or exploding gradients and by promoting consistent feature distributions, particularly in non-stationary settings. However, to the best of our knowledge, a formal theoretical analysis of how the choice of normalization affects model behavior and expressivity remains absent from the literature. As a result, the specific normalization strategy often varies across implementations and is typically chosen at the discretion of the user. 

In this work, we aim to study the impact of such normalization procedures through the lens of expressivity, as defined in Definition~\ref{def:expressivity}. While many preprocessing strategies exist, we focus on two widely used transformations:
\begin{itemize}
    \item \textbf{Standard Normalization}: Rescales each feature using a mean and standard deviation, which are computed either for each channel (instance-based) or globally.
    \item \textbf{Min-Max Normalization}: Rescales each feature linearly to the $[0, 1]$ interval using the minimum and maximum values, which are computed either for each channel (instance-based) or all the channels (global).
\end{itemize}

\noindent Global normalization ensures comparability across channels by preserving their scale ratios. This is beneficial in settings where the absolute scale of a channel conveys meaningful information. Conversely, instance-based normalization is often preferable when intra-channel variation is more important. We now present a theoretical analysis showing how each normalization strategy influences model expressivity.

\begin{theorem}\label{theo:standard_normalization}
Let $f \colon \mathcal{X} \subseteq \mathbb{R}^{n \times d} \rightarrow \mathcal{Y} \subseteq \mathbb{R}^d$ be a Transformer-based model (TBM) under the setting described in our problem formulation. Then:
\begin{itemize}
    \item If $f$ uses \textbf{Instance-based Standard Normalization}, then $f$ is $(\epsilon,\sigma,\gamma)$-expressive with:
    \[
    \gamma = \left(\sqrt{\sum_{i=1}^{N} \frac{1}{v_i^2}} \right)\frac{\epsilon}{\sigma} \left(\frac{d}{d-1}\right)^2 (1+C_1) C_2,
    \]
    \item If $f$ uses \textbf{Global Standard Normalization}, then $f$ is $(\epsilon,\sigma,\gamma)$-expressive with:
    \[
    \gamma = \frac{\epsilon\sqrt{d} }{ \sigma v}  \left(\frac{d}{d-1}\right)^2 (1+C_1) C_2,
    \]
\end{itemize}
\begin{align*} 
    \text{with:} \hspace{0.5em} &  C_1 = \lVert W_O \lVert \sqrt{H} 
    \max_h \Biggl[ \Bigl\lVert W^{V,h} \Bigr\rVert
    \left(4\,\frac{n}{\sqrt{d/H}}\,\Bigl\lVert W^{Q,h} \Bigr\rVert\,\Bigl\lVert W^{K,h} \Bigr\rVert + 1\right) \Biggr], \\
    & C_2 = 1 + \Bigl\lVert W_{FFN} \Bigr\rVert.
\end{align*}
\end{theorem}

Theorem~\ref{theo:standard_normalization} provides expressivity bounds under standard normalization. Due to the translation invariance of self-attention, the mean component value does not influence the bound; instead, the key controlling factor is the variance. In the global setting, the chosen global variance directly influences expressivity. If the variance is dominated by one high-scale channel, other channels may be diminished, thus reducing the model's sensitivity to them. This effect suggests that global normalization may be suitable only when a dominant channel is semantically more important; otherwise, instance-level normalization should be preferred. Similar insights are also seen in the case of the Min-Max scaling strategy.

\begin{theorem}\label{theo:minmax_normalization}
Let $f \colon \mathcal{X} \subseteq \mathbb{R}^{n \times d} \rightarrow \mathcal{Y} \subseteq \mathbb{R}^d$ be a TBM under our defined setup. Then:
\begin{itemize}
    \item If $f$ uses \textbf{Instance-based Min-Max Normalization}, then $f$ is $(\epsilon,\sigma,\gamma)$-expressive with:
    \[
    \gamma = \left(\sqrt{\sum_{i=1}^{N} \frac{1}{(x_i^{\max} - x_i^{\min})^2}} \right)\frac{\epsilon}{\sigma} \left(\frac{d}{d-1}\right)^2 (1+C_1) C_2,
    \]
    \item If $f$ uses \textbf{Global Min-Max Normalization}, then $f$ is $(\epsilon,\sigma,\gamma)$-expressive with:
    \[
    \gamma = \frac{\epsilon\sqrt{d} }{ \sigma (x^{\max} - x^{\min})}  \left(\frac{d}{d-1}\right)^2 (1+C_1) C_2.
    \]
\end{itemize}
\begin{align*} 
    \text{with:} \hspace{0.5em} & C_1 = 1 + \lVert W_O \lVert \sqrt{H} 
    \max_h \Biggl[ \Bigl\lVert W^{V,h} \Bigr\rVert
    \left(4\,\frac{n}{\sqrt{d/H}}\,\Bigl\lVert W^{Q,h} \Bigr\rVert\,\Bigl\lVert W^{K,h} \Bigr\rVert + 1\right) \Biggr], \\
    & C_2 = 1 + \Bigl\lVert W_{FFN} \Bigr\rVert.
\end{align*}
\end{theorem}

\noindent Across both normalization types, we observe that expressivity depends on both the model's weight norms and the scaling range induced by the normalization method. 

\noindent \textbf{Main Takeaways.} Our theoretical results yield several important insights into the influence of normalization strategies on the expressivity of Transformer-based models for time series. When using global normalization, the model’s sensitivity is primarily dictated by the dominant channels in terms of scale, which may suppress useful information from less prominent channels. This can be problematic in settings where fine-grained, multi-channel variation is essential. In such cases, instance-based normalization is more appropriate, as it preserves the local scale and dynamics of each channel independently. Conversely, if the downstream task depends heavily on identifying large variations in a specific channel, such as in anomaly detection, where a spike in one sensor may signal a fault — global normalization may enhance expressivity by amplifying the dominant signal. However, in tasks like motion recognition, where multiple channels (e.g., different spatial axes) carry complementary information, global normalization may obscure critical patterns, and instance-based scaling becomes preferable.

\noindent Full proofs of all theorems can be found in Appendix \ref{appendix:proof_standard} and Appendix \ref{appendix:proof_minmax}.

\renewcommand{\arraystretch}{2.0}
\setlength{\tabcolsep}{3pt}

\begin{table*}[h]
\caption{Mean (and corresponding standard deviation) Accuracy of the different considered pre-processing methods for different models and datasets. "None" corresponds to the case where the time series is used without normalization. }
\label{tab:results_classification}
\resizebox{\textwidth}{!}{%
\begin{tabular}{llccccccccc}
\hline
                Model        & Normalization & EthanolConcentration & FaceDetection    & Handwriting      & Heartbeat        & JapaneseVowels   & PEMS-SF          & SelfRegulationSCP1 & SpokenArabicDigits & UWaveGestureLibrary \\ \hline
\multirow{6}{*}{\rotatebox{90}{Transformer}} & Standard (Instance)      & $28.77 \pm 0.90$     & $67.21 \pm 0.47$ & $37.57 \pm 0.53$ & $73.01 \pm 0.46$ & $79.82 \pm 0.34$ & $88.05 \pm 2.68$ & $85.44 \pm 1.86$   & $98.27 \pm 0.26$   & $86.77 \pm 0.39$    \\
                             & Standard (Global)      & $31.05 \pm 0.47$     & $67.45 \pm 0.68$ & $37.18 \pm 0.29$ & $77.24 \pm 0.61$ & $97.84 \pm 0.22$ & $85.36 \pm 1.79$ & $90.10 \pm 0.28$   & $98.42 \pm 0.09$   & $86.77 \pm 0.39$    \\ \cline{2-11} 
                             & None          & $30.93 \pm 0.72$     & $67.91 \pm 0.28$ & $37.57 \pm 0.53$ & $74.80 \pm 0.46$ & $98.11 \pm 0.00$ & $90.56 \pm 0.55$ & $91.81 \pm 0.74$   & $98.59 \pm 0.04$   & $86.77 \pm 0.39$    \\
                             & Minmax        & $29.91 \pm 1.53$     & $50.01 \pm 0.01$ & $27.80 \pm 0.84$ & $75.12 \pm 0.80$ & $97.48 \pm 0.34$ & $87.48 \pm 0.72$ & $91.70 \pm 1.13$   & $98.21 \pm 0.29$   & $83.75 \pm 0.68$    \\
                             & Quantile      & $36.12 \pm 2.54$     & $67.39 \pm 0.57$ & $30.12 \pm 0.88$ & $77.40 \pm 0.83$ & $97.12 \pm 0.46$ & $87.48 \pm 1.66$ & $89.42 \pm 1.28$   & $98.42 \pm 0.17$   & $87.08 \pm 0.15$    \\
                             & Robust        & $32.07 \pm 2.68$     & $68.11 \pm 0.75$ & $37.45 \pm 0.84$ & $75.77 \pm 1.28$ & $97.57 \pm 0.44$ & $82.66 \pm 1.42$ & $91.24 \pm 0.90$   & $98.54 \pm 0.20$   & $86.67 \pm 0.39$    \\ \hline
\multirow{6}{*}{\rotatebox{90}{PatchTST}}    & Standard (Instance)      & $26.49 \pm 0.95$     & $67.39 \pm 1.38$ & $27.88 \pm 1.01$ & $71.54 \pm 1.00$ & $70.36 \pm 1.00$ & $84.39 \pm 0.47$ & $85.55 \pm 1.58$   & $96.65 \pm 0.44$   & $85.52 \pm 0.74$    \\
                             & Standard (Global)      & $27.25 \pm 0.65$     & $67.61 \pm 0.38$ & $27.88 \pm 1.01$ & $72.36 \pm 1.61$ & $95.68 \pm 0.22$ & $83.62 \pm 1.09$ & $85.78 \pm 1.68$   & $96.50 \pm 0.30$   & $85.52 \pm 0.74$    \\ \cline{2-11} 
                             & None          & $26.36 \pm 0.36$     & $67.46 \pm 0.84$ & $27.88 \pm 1.01$ & $72.20 \pm 0.40$ & $91.53 \pm 0.46$ & $82.85 \pm 2.13$ & $85.67 \pm 1.47$   & $96.29 \pm 0.43$   & $85.52 \pm 0.74$    \\
                             & MinMax        & $26.48 \pm 0.64$     & $68.14 \pm 0.52$ & $27.88 \pm 1.01$ & $72.35 \pm 0.45$ & $84.52 \pm 0.14$ & $85.35 \pm 1.78$ & $85.89 \pm 1.81$   & $97.21 \pm 0.15$   & $85.52 \pm 0.74$    \\
                             & Quantile      & $34.22 \pm 0.62$     & $67.42 \pm 0.18$ & $26.43 \pm 0.73$ & $72.68 \pm 0.80$ & $95.32 \pm 0.34$ & $82.66 \pm 1.70$ & $84.53 \pm 0.70$   & $96.74 \pm 0.15$   & $85.83 \pm 0.29$    \\
                             & Robust        & $27.12 \pm 1.09$     & $68.31 \pm 0.14$ & $27.88 \pm 1.01$ & $71.38 \pm 1.22$ & $95.14 \pm 0.22$ & $82.85 \pm 2.72$ & $85.55 \pm 1.58$   & $96.70 \pm 0.49$   & $85.52 \pm 0.74$    \\ \hline
\end{tabular}%
}
\end{table*}

\section{Experimental Validation}

In this section, we empirically validate the theoretical insights derived in the previous sections using real-world time series datasets and multiple Transformer-based architectures. We start by outlining our experimental setup and then report and analyze the different results.

\begin{figure*}[h]
    \centering
    \includegraphics[width=\linewidth]{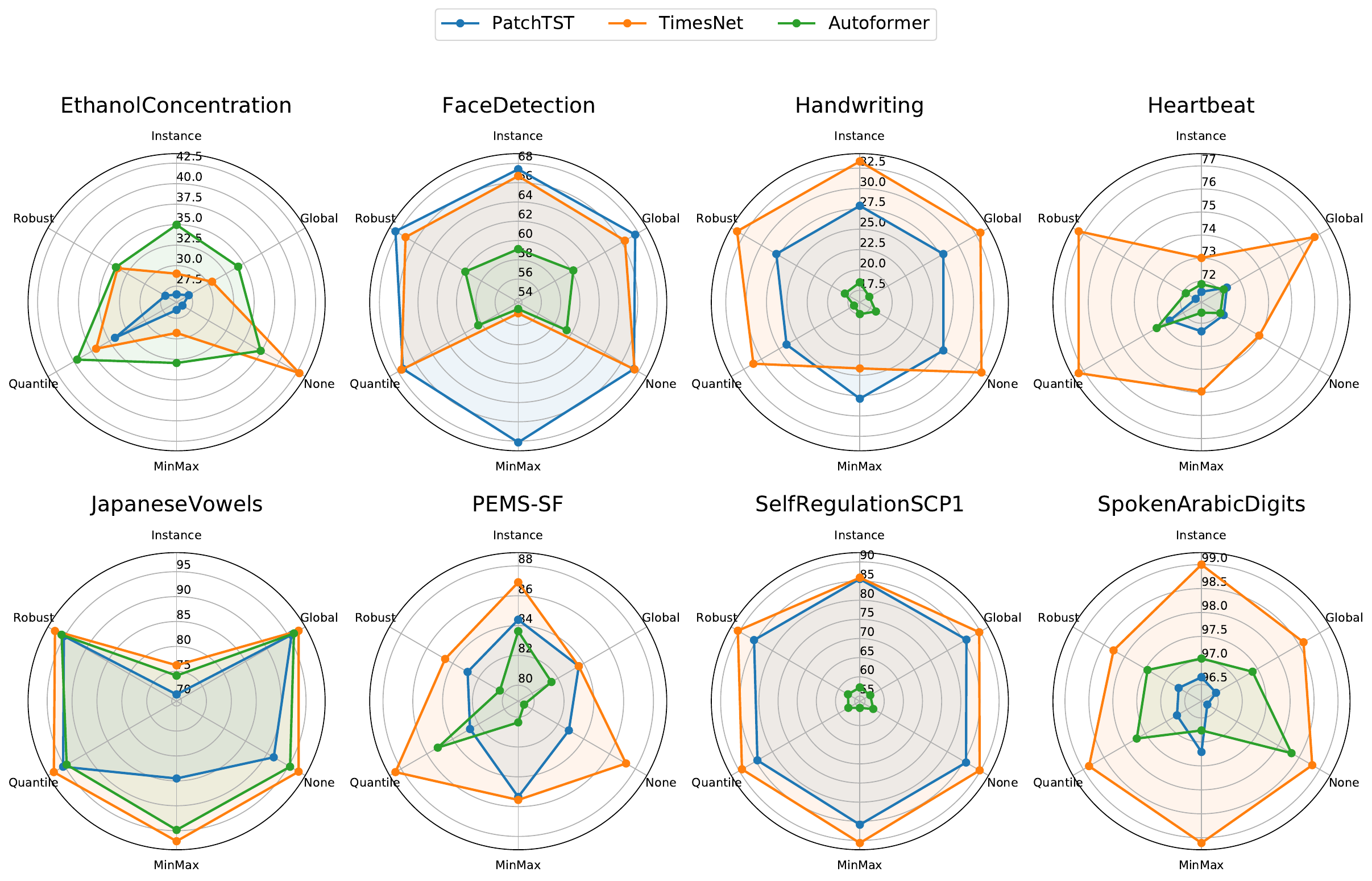}
    \caption{Resulting Accuracy Comparison of the different considered pre-processing methods for different models and datasets for the classification task. "None" corresponds to the case where the time series is used without normalization.}
    \label{fig:radar_chart_results}
\end{figure*}

\subsection{Experimental Setup}

We focus on two common downstream tasks in time series representation learning: classification and forecasting. In the classification setting, the goal is to predict the class label associated with a given time series instance. In the forecasting task, the objective is to predict future values based on a partially observed sequence.

\noindent While our theoretical analysis centers on a vanilla Transformer architecture, we hypothesize that the derived insights generalize to more advanced adaptations. To test this, we evaluate the effect of different normalization strategies across several architectures, including a standard Transformer, PatchTST~\citep{nie2023patchtst}, Autoformer~\citep{wu2021autoformer}, and TimesNet~\citep{wu2023timesnet}.

\noindent For classification, we use benchmark datasets from the UEA time series classification repository~\citep{bagnall2018uea}. Each experiment is repeated five times to mitigate the effects of stochasticity due to random initialization. All models are trained using the Adam optimizer. For classification, we optimize the cross-entropy loss, whereas for forecasting, we use the $\ell_2$ loss and report the Mean Absolute Error (MAE).

\noindent Additional implementation details, including architectural, optimization hyperparameters, and details about the datasets, are provided in the appendix.

\subsection{Experimental Results — Classification}\label{sec:experimental_results_classification}

Table~\ref{tab:results_classification} and Figure~\ref{fig:radar_chart_results} reports the classification accuracy and standard deviation across different models and normalization strategies for several benchmark datasets.

\noindent  We start by validating our theoretical insights regarding the distinction between instance-based and global normalization under the Standard scaling method. While global normalization often yields better performance overall, we observe notable exceptions. For example, on datasets such as \textit{PEMS-SF} and \textit{Handwriting}, instance-based normalization outperforms its global counterpart. These cases align with our theoretical claims: when specific local channel dynamics are important, global normalization can obscure critical features, leading to degraded performance. Conversely, on datasets like \textit{UWaveGestureLibrary}, where the inputs are already normalized by design, all strategies, including global, instance-based, and no normalization, yield comparable results.

\noindent To deepen the analysis, we compare the performance of different normalization techniques beyond just instance versus global. The results reveal that no single normalization strategy universally outperforms the others across all models and datasets. Instead, the most effective approach appears to be highly context-dependent. In particular, there are cases where omitting normalization altogether (denoted as \textsc{None}) leads to the best performance. This finding is especially striking, as it challenges the common assumption that normalization is always beneficial.

\noindent In summary, our experimental findings corroborate the theoretical analysis: preprocessing choices have a significant impact on downstream performance, and selecting the appropriate strategy requires consideration of both the dataset characteristics and the task-specific objectives.

\noindent While we have chosen to represent the results in some model in form of Radar Chart, the full results are also presented in a Table-like format in Appendix \ref{appendix:additional_results}. 

\begin{table}[h]
\caption{MAE  (and corresponding standard deviation)  for each model under different normalization strategies for long-term forecasting task on ETT-Small Dataset.}
\label{tab:table_forecasting}
\centering
\resizebox{0.7\columnwidth}{!}{%
\begin{tabular}{lccc}
\hline
\textbf{Normalization} & \textbf{Transformer} & \textbf{Autoformer} & \textbf{PatchTST} \\ \hline
Instance (Standard)    & 1.0023 ± 0.0041      & 0.9818 ± 0.0281     & 0.8010 ± 0.0029   \\
Global (Standard)      & 0.7617 ± 0.0512      & 0.4593 ± 0.0023     & 0.4046 ± 0.0004   \\ \hline
None                   & 2.9003 ± 0.1844      & 1.7321 ± 0.0362     & 1.4864 ± 0.0047   \\
MinMax                 & 0.7912 ± 0.0010       & 0.6512 ± 0.0020      & 0.5819 ± 0.0001    \\
Quantile               & 0.7717 ± 0.0339      & 0.4746 ± 0.0181     & 0.4184 ± 0.0015   \\
Robust                 & 0.7327 ± 0.0879      & 0.4583 ± 0.0039     & 0.4118 ± 0.0020   \\ \hline
\end{tabular}%
}
\end{table}

\subsection{Experimental Results — Forecasting}

\noindent For the forecasting task, we evaluate the different normalization strategies using the ETT-Small dataset. Table~\ref{tab:table_forecasting} reports the mean MAE along with the corresponding standard deviation across different models and preprocessing configurations. 

\begin{wrapfigure}{r}{0.68\linewidth}
    \centering
    \includegraphics[width=\linewidth]{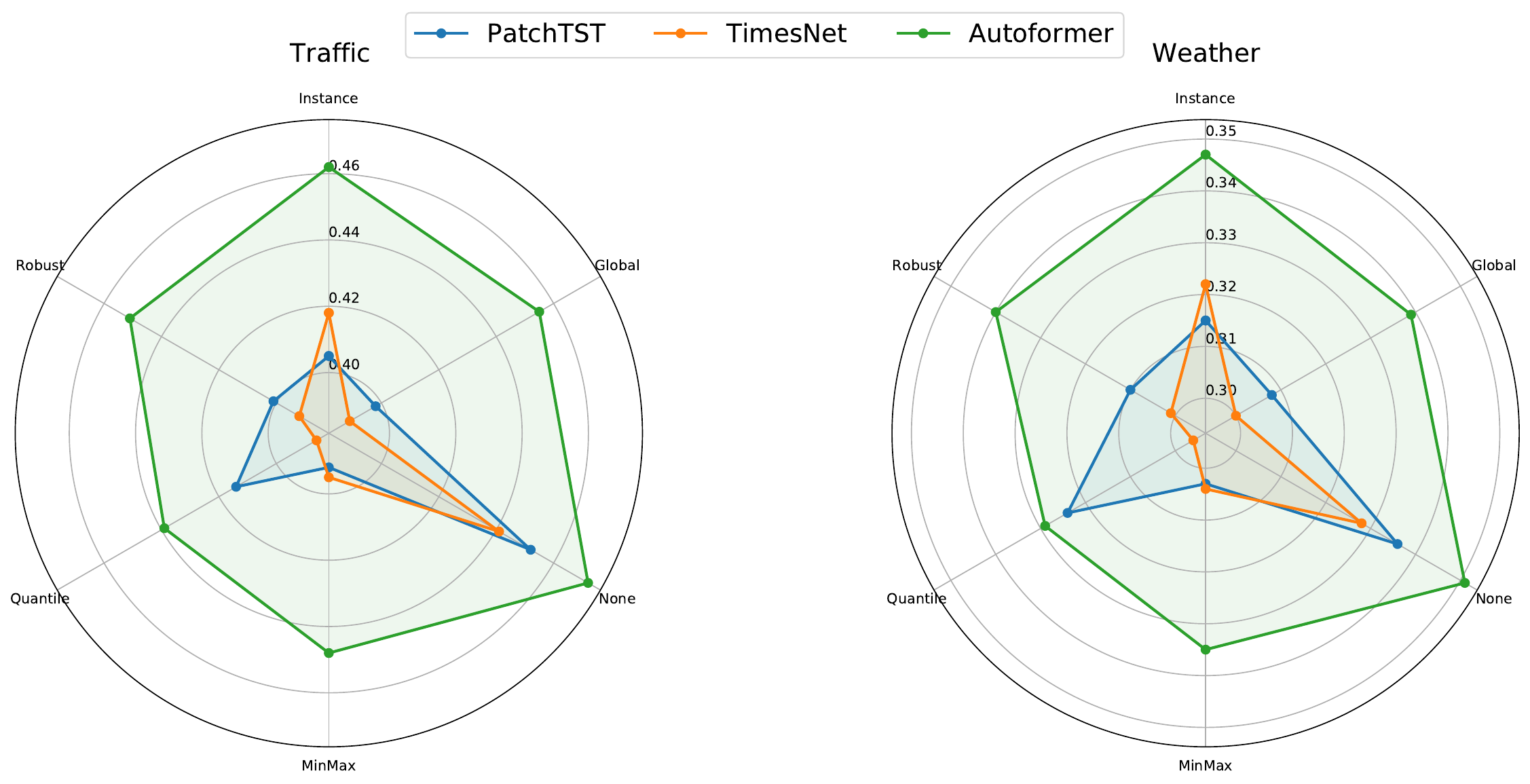}
    \caption{Resulting MAE of the different considered pre-processing methods for different models and datasets for the forecasting task. ``None'' corresponds to the case where the time series is used without normalization.}
    \label{fig:Figure_Forecasting}
\end{wrapfigure}

Interestingly, for this particular task, we observe that global normalization under the Standard scaling method consistently outperforms its instance-based counterpart. This suggests that, in the context of forecasting, preserving global scale information may be more beneficial, possibly because the temporal dependencies being modeled rely on longer-range trends rather than fine-grained, channel-specific variations.

\noindent Nevertheless, and in line with our observations from the classification task, no single normalization strategy emerges as universally optimal. This highlights once again the sensitivity of Transformer-based models to preprocessing choices and emphasizes the importance of selecting the normalization method in accordance with the specific characteristics of the dataset and task.

\noindent These results further support our central thesis: normalization is not merely a routine preprocessing step, but a critical design choice that can significantly affect the expressivity and effectiveness of time series models.

\noindent Similar trends can be observed across other forecasting datasets, including Traffic and Weather. As illustrated in Figure~\ref{fig:Figure_Forecasting}, the choice of pre-processing method has a noticeable impact on forecasting performance. In particular, we observe clear variations in MAE across the different approaches, highlighting the importance of careful data preparation.

\section{Conclusion}

In this work, we investigated the role of input normalization, an often-overlooked but ubiquitous preprocessing step, in Transformer-based time series representation models. Our primary contribution is a theoretical and empirical examination of how different normalization strategies, particularly instance-based versus global scaling, affect the expressivity and performance of such models.

\noindent Our findings demonstrate that there is no universal normalization strategy that guarantees optimal performance across all tasks and datasets. Instead, the choice of normalization should be guided by the underlying characteristics of the data and the requirements of the downstream task. Our theoretical framework provides an interpretable lens through which to understand how scaling choices propagate through the model, while our empirical results confirm the practical impact of these decisions. In short, there is no "free lunch" in preprocessing; careful selection of normalization strategies is essential for robust and effective time series modeling.

\noindent \textbf{Limitations and Future Work.} While we do not propose a new normalization method that performs uniformly well across all settings, we provide a theoretical foundation and empirical observations that could inform the design of more adaptive or task-aware scaling approaches. We believe that this study opens promising directions for future research on principled preprocessing methods that could further enhance the performance and generalization capabilities of time series models.

\bibliographystyle{plainnat}
\bibliography{references}

\newpage
\appendix
\vbox{%
\hsize\linewidth
\linewidth\hsize
\vskip 0.1in
\centering
{\LARGE\bf Appendix: Be Wary of Your Time Series Preprocessing\par}
\vspace{2\baselineskip}
}

\section{Proof Of Theorem \ref{theo:standard_normalization}}\label{appendix:proof_standard}

\begin{theorem}
Let $f \colon \mathcal{X} \subseteq \mathbb{R}^{n \times d} \rightarrow \mathcal{Y} \subseteq \mathbb{R}^d$ be a Transformer-based model (TBM) under the setting described in our problem formulation. Then:
\begin{itemize}
    \item If $f$ uses \textbf{Instance-based Standard Normalization}, then $f$ is $(\epsilon,\sigma,\gamma)$-expressive with:
    \[
    \gamma = \left(\sqrt{\sum_{i=1}^{N} \frac{1}{v_i^2}} \right)\frac{\epsilon}{\sigma} \left(\frac{d}{d-1}\right)^2 (1+C_1) C_2,
    \]
    \item If $f$ uses \textbf{Global Standard Normalization}, then $f$ is $(\epsilon,\sigma,\gamma)$-expressive with:
    \[
    \gamma = \frac{\epsilon\sqrt{d} }{ \sigma v}  \left(\frac{d}{d-1}\right)^2 (1+C_1) C_2,
    \]
\end{itemize}
\begin{align*} 
    \text{with:} \hspace{0.5em} &  C_1 = \lVert W_O \lVert \sqrt{H} 
    \max_h \Biggl[ \Bigl\lVert W^{V,h} \Bigr\rVert
    \left(4\,\frac{n}{\sqrt{d/H}}\,\Bigl\lVert W^{Q,h} \Bigr\rVert\,\Bigl\lVert W^{K,h} \Bigr\rVert + 1\right) \Biggr], \\
    & C_2 = 1 + \Bigl\lVert W_{FFN} \Bigr\rVert.
\end{align*}
\end{theorem}

\begin{proof}

Let $g$ be a transformer-based model following the considered problem setup. Specifically, let $g$ be a $1$-Layer TBM. 

\noindent Let's now consider an input $X \in \mathcal{X}$ composed of $n$ tokens $x_i \in \mathbb{R}^{d}$. We consider that our model $f$ is built using the dot-product self-attention which can be formulated as:
\begin{align*}
    \text{AH}(x) &= \mathrm{Softmax}(\frac{(XW^Q)(XW^K)^T}{\sqrt{\frac{D}{H}}})XW^V \\
    & = PXW^V = h(X) W^V,
\end{align*}

\noindent In this following part, we consider that the pre-processing is done using a standard normalization, which consists of taking into account the mean and the variance. We consider the specific case of instance-based, which afterwards easily be adapted for the global case:

$$h(x_i) = \frac{x_i - \mu_i}{v_i},$$ 

\noindent with $\mu_i$ being the mean value and $v_i$ the corresponding variance. 

\noindent The final model $f$ can be therefore be written as the composition of the two functions:

\begin{equation}\label{eq:union_function}
    f(X) = g \circ h (X)
\end{equation}

\noindent Consequently, the final expressivity bound of $f$ is simply the multiplication of the two bounds.

\noindent Let's consider a similar element to the considered input $X$, denoted as $\tilde{X} \in \mathcal{N}_{\epsilon}(X)$, based on previous work \citep{ennadir2025poolwiselyeffectpooling}, we have the following:

\begin{align*}
    \lVert  g(X) - g(\tilde{X})  \lVert \leq \left(\frac{d}{d-1}\right)^2 (1+C_1) C_2 \epsilon,
    \end{align*}
\begin{align*} 
    & \text{with:} \hspace{0.5em}  C_1 = \lVert W_O \lVert \sqrt{H} 
    \max_h \Biggl[ \Bigl\lVert W^{V,h} \Bigr\rVert \\
    &\hspace{7em}\left(4\,\frac{n}{\sqrt{d/H}}\,\Bigl\lVert W^{Q,h} \Bigr\rVert\,\Bigl\lVert W^{K,h} \Bigr\rVert + 1\right) \Biggr], \\
    & C_2 = 1 + \Bigl\lVert W_{FFN} \Bigr\rVert.
\end{align*}

\noindent Let's now consider the effect of normalization, we have the following:

\begin{align*}
    \lVert h(X) - h(\tilde{X}) \lVert & \leq 
    \begin{bmatrix}
        \frac{x_1 - \mu_1}{v_1} - \frac{{\tilde{x}}_1 - \mu_1}{v_1}\\
        \ldots \\
         \frac{x_d - \mu_d}{v_d} - \frac{{\tilde{x}}_d - \mu_d}{v_d}
    \end{bmatrix}
    \leq
        \begin{bmatrix}
        \frac{x_1 - {\tilde{x}}_1}{v_1}\\
        \ldots \\
         \frac{x_d - {\tilde{x}}_d}{v_d}
    \end{bmatrix} \\
    & \leq \left(\sqrt{\sum_{i=1}^{d} \frac{1}{v_i^2}} \right) \epsilon 
\end{align*}

\noindent Given decomposition of the function $f$ in Equation \ref{eq:union_function}, we can write:

\begin{align*}
    \lVert f(X) - f(\tilde{X}) \lVert \leq \left(\sqrt{\sum_{i=1}^{d} \frac{1}{v_i^2}} \right)  \left(\frac{d}{d-1}\right)^2 (1+C_1) C_2 \epsilon
\end{align*}

\noindent For the Global-based scaling, we use the same values for the scaling parameters which can be written as: $\forall i \leq d, v_i = v$

\noindent In this perspective, we can directly deduce from the instance-based derived bound that:
\begin{align*}
    \lVert f(X) - f(\tilde{X}) \lVert & \leq \left(\sqrt{\sum_{i=1}^{d} \frac{1}{v^2}} \right)  \left(\frac{d}{d-1}\right)^2 (1+C_1) C_2 \epsilon \\
    & \leq \frac{\sqrt{d}}{v}  \left(\frac{d}{d-1}\right)^2 (1+C_1) C_2 \epsilon.
\end{align*}

\noindent For both inequalities, by taking the Markov Inequality, we get the desired bound, specifically:

\begin{itemize}
    \item For Global-based scaling: $$\gamma = \frac{\epsilon\sqrt{d} }{ \sigma v}  \left(\frac{d}{d-1}\right)^2 (1+C_1) C_2.$$
    \item For Instance-based scaling: $$\gamma =  \left(\sqrt{\sum_{i=1}^{N} \frac{1}{v_i^2}} \right)\frac{\epsilon}{\sigma} \left(\frac{d}{d-1}\right)^2 (1+C_1) C_2,$$
\end{itemize}

\begin{align*} 
    & \text{with:} \hspace{0.5em}  C_1 = \lVert W_O \lVert \sqrt{H} 
    \max_h \Biggl[ \Bigl\lVert W^{V,h} \Bigr\rVert \\
    &\hspace{7em}\left(4\,\frac{n}{\sqrt{d/H}}\,\Bigl\lVert W^{Q,h} \Bigr\rVert\,\Bigl\lVert W^{K,h} \Bigr\rVert + 1\right) \Biggr], \\
    & C_2 = 1 + \Bigl\lVert W_{FFN} \Bigr\rVert.
\end{align*}

\end{proof}

\section{Proof Of Theorem \ref{theo:minmax_normalization}}\label{appendix:proof_minmax}

\begin{theorem}
Let $f \colon \mathcal{X} \subseteq \mathbb{R}^{n \times d} \rightarrow \mathcal{Y} \subseteq \mathbb{R}^d$ be a TBM under our defined setup. Then:
\begin{itemize}
    \item If $f$ uses \textbf{Instance-based Min-Max Normalization}, then $f$ is $(\epsilon,\sigma,\gamma)$-expressive with:
    \[
    \gamma = \left(\sqrt{\sum_{i=1}^{N} \frac{1}{(x_i^{\max} - x_i^{\min})^2}} \right)\frac{\epsilon}{\sigma} \left(\frac{d}{d-1}\right)^2 (1+C_1) C_2,
    \]
    \item If $f$ uses \textbf{Global Min-Max Normalization}, then $f$ is $(\epsilon,\sigma,\gamma)$-expressive with:

    \[
    \gamma = \frac{\epsilon\sqrt{d} }{ \sigma (x^{\max} - x^{\min})}  \left(\frac{d}{d-1}\right)^2 (1+C_1) C_2.
    \]
\end{itemize}
\begin{align*}
    \text{with:} \hspace{0.5em} & C_1 = 1 + \lVert W_O \lVert \sqrt{H} 
    \max_h \Biggl[ \Bigl\lVert W^{V,h} \Bigr\rVert \left(4\,\frac{n}{\sqrt{d/H}}\,\Bigl\lVert W^{Q,h} \Bigr\rVert\,\Bigl\lVert W^{K,h} \Bigr\rVert + 1\right) \Biggr], \\
    & C_2 = 1 + \Bigl\lVert W_{FFN} \Bigr\rVert.
\end{align*}
\end{theorem}

\begin{proof}
\noindent Similar to the previous proof, we follow the same setup, where we consider that $f$ is a $1$-Layer TBM. In this case, we rather consider the MinMax scaling, which can be formulated for each channel $i$ in the case of instance-based as:

$$h_i(x_i) = \frac{x_i - x_{min}}{x_i^{\max} - x_i^{\min}}$$

\noindent Similar to the previous proof, the function can be written as a decomposition of two elements: 
$$f(X) = g \circ h(X)$$

\noindent In the case of MinMax scaling, we can write the following:
\begin{align*}
    \lVert h(X) - h(\tilde{X}) \lVert & \leq 
    \begin{bmatrix}
        \frac{x_1 - x_1^{min}}{x_1^{\max} - x_1^{\min}} - \frac{{\tilde{x}}_1 - x_1^{min}}{x_1^{\max} - x_1^{\min}}\\
        \ldots \\
         \frac{x_d - x_d^{min}}{x_d^{\max} - x_d^{\min}} - \frac{{\tilde{x}}_d - x_d^{min}}{x_d^{\max} - x_d^{\min}}
    \end{bmatrix} \\
    & \leq
        \begin{bmatrix}
        \frac{x_1 - {\tilde{x}}_1}{x_1^{\max} - x_1^{\min}}\\
        \ldots \\
         \frac{x_d - {\tilde{x}}_d}{x_d^{\max} - x_d^{\min}}
    \end{bmatrix} \\
    & \leq \left(\sqrt{\sum_{i=1}^{d} \frac{1}{(x_i^{\max} - x_i^{\min})^2}} \right) \epsilon 
\end{align*}

\noindent Given decomposition of the function $f$ in Equation \ref{eq:union_function}, we can write:

\begin{center}
$
    \lVert f(X) - f(\tilde{X}) \lVert \leq \left(\sqrt{\sum_{i=1}^{d} \frac{1}{(x_i^{\max} - x_i^{\min})^2}} \right)  \left(\frac{d}{d-1}\right)^2 (1+C_1) C_2 \epsilon
$
\end{center}

\noindent Similar to the previous proof, in the case of global context where the same values of $x^{max}$ and $x^{min}$ are used, we can write:

\begin{center}
$
\begin{aligned}
\lVert f(X) - f(\tilde{X}) \rVert 
&\leq 
\left(\sqrt{\sum_{i=1}^{d} \frac{1}{(x^{\max} - x^{\min})^2}} \right)
\left(\frac{d}{d-1}\right)^2 (1+C_1) C_2 \epsilon \\
&\leq 
\frac{\sqrt{d}}{(x^{\max} - x^{\min})}
\left(\frac{d}{d-1}\right)^2 (1+C_1) C_2 \epsilon
\end{aligned}
$
\end{center}

\renewcommand{\arraystretch}{2}
\setlength{\tabcolsep}{2pt}

\begin{table*}[h]
\caption{Mean (and corresponding standard deviation) Accuracy of the different considered pre-processing methods for different models and datasets. "None" corresponds to the case where the time series is used without normalization. }
\label{tab:results_classification_appendix}
\resizebox{\textwidth}{!}{%
\begin{tabular}{llccccccccc}
\hline
Model                        & Normalization & EthanolConcentration & FaceDetection    & Handwriting      & Heartbeat        & JapaneseVowels   & PEMS-SF          & SelfRegulationSCP1 & SpokenArabicDigits & UWaveGestureLibrary \\ \hline
\multirow{6}{*}{\rotatebox{90}{Transformer}} & Standard (Instance)      & $28.77 \pm 0.90$     & $67.21 \pm 0.47$ & $37.57 \pm 0.53$ & $73.01 \pm 0.46$ & $79.82 \pm 0.34$ & $88.05 \pm 2.68$ & $85.44 \pm 1.86$   & $98.27 \pm 0.26$   & $86.77 \pm 0.39$    \\
                             & Standard (Global)      & $31.05 \pm 0.47$     & $67.45 \pm 0.68$ & $37.18 \pm 0.29$ & $77.24 \pm 0.61$ & $97.84 \pm 0.22$ & $85.36 \pm 1.79$ & $90.10 \pm 0.28$   & $98.42 \pm 0.09$   & $86.77 \pm 0.39$    \\ \cline{2-11} 
                             & None          & $30.93 \pm 0.72$     & $67.91 \pm 0.28$ & $37.57 \pm 0.53$ & $74.80 \pm 0.46$ & $98.11 \pm 0.00$ & $90.56 \pm 0.55$ & $91.81 \pm 0.74$   & $98.59 \pm 0.04$   & $86.77 \pm 0.39$    \\
                             & Minmax        & $29.91 \pm 1.53$     & $50.01 \pm 0.01$ & $27.80 \pm 0.84$ & $75.12 \pm 0.80$ & $97.48 \pm 0.34$ & $87.48 \pm 0.72$ & $91.70 \pm 1.13$   & $98.21 \pm 0.29$   & $83.75 \pm 0.68$    \\
                             & Quantile      & $36.12 \pm 2.54$     & $67.39 \pm 0.57$ & $30.12 \pm 0.88$ & $77.40 \pm 0.83$ & $97.12 \pm 0.46$ & $87.48 \pm 1.66$ & $89.42 \pm 1.28$   & $98.42 \pm 0.17$   & $87.08 \pm 0.15$    \\
                             & Robust        & $32.07 \pm 2.68$     & $68.11 \pm 0.75$ & $37.45 \pm 0.84$ & $75.77 \pm 1.28$ & $97.57 \pm 0.44$ & $82.66 \pm 1.42$ & $91.24 \pm 0.90$   & $98.54 \pm 0.20$   & $86.67 \pm 0.39$    \\ \hline
\multirow{6}{*}{\rotatebox{90}{Autoformer}}  & Standard (Instance)      & $34.98 \pm 0.82$     & $59.11 \pm 0.63$ & $18.47 \pm 0.51$ & $71.87 \pm 0.46$ & $74.14 \pm 0.56$ & $83.62 \pm 0.72$ & $57.22 \pm 1.13$   & $97.04 \pm 0.16$   & $51.35 \pm 1.18$    \\
                             & Standard (Global)      & $34.22 \pm 1.42$     & $60.20 \pm 2.18$ & $17.37 \pm 0.71$ & $72.20 \pm 0.00$ & $96.22 \pm 0.38$ & $81.50 \pm 0.82$ & $56.77 \pm 0.58$   & $97.38 \pm 0.48$   & $51.35 \pm 1.18$    \\ \cline{2-11} 
                             & None          & $37.39 \pm 5.52$     & $59.40 \pm 2.90$ & $18.31 \pm 1.02$ & $72.03 \pm 0.23$ & $95.32 \pm 1.50$ & $79.38 \pm 1.44$ & $57.68 \pm 0.74$   & $98.32 \pm 0.17$   & $51.35 \pm 1.18$    \\
                             & Minmax        & $32.95 \pm 0.78$     & $54.30 \pm 1.10$ & $17.49 \pm 0.44$ & $71.54 \pm 0.92$ & $94.86 \pm 0.66$ & $80.35 \pm 1.63$ & $55.40 \pm 1.81$   & $96.76 \pm 0.71$   & $45.31 \pm 0.77$    \\
                             & Quantile      & $39.54 \pm 1.55$     & $58.39 \pm 0.90$ & $16.86 \pm 0.22$ & $73.33 \pm 1.00$ & $94.50 \pm 0.78$ & $85.16 \pm 0.72$ & $57.11 \pm 0.64$   & $97.71 \pm 0.40$   & $55.73 \pm 0.59$    \\
                             & Robust        & $34.09 \pm 1.18$     & $59.94 \pm 0.99$ & $18.16 \pm 0.53$ & $71.87 \pm 0.46$ & $95.68 \pm 0.96$ & $80.35 \pm 1.70$ & $57.22 \pm 1.32$   & $97.45 \pm 0.42$   & $52.92 \pm 1.41$    \\ \hline
\multirow{6}{*}{\rotatebox{90}{TimesNet}}    & Standard (Instance)      & $29.02 \pm 1.40$     & $66.69 \pm 0.27$ & $33.33 \pm 0.87$ & $73.01 \pm 0.46$ & $76.22 \pm 1.15$ & $86.90 \pm 2.23$ & $85.89 \pm 1.40$   & $99.00 \pm 0.17$   & $88.02 \pm 0.64$    \\
                             & Standard (Global)      & $30.54 \pm 1.25$     & $66.38 \pm 0.23$ & $33.14 \pm 1.35$ & $76.75 \pm 1.00$ & $97.30 \pm 0.00$ & $83.62 \pm 1.19$ & $89.65 \pm 0.98$   & $98.61 \pm 0.39$   & $88.02 \pm 0.64$    \\ \cline{2-11} 
                             & None          & $42.84 \pm 1.59$     & $67.54 \pm 0.28$ & $33.29 \pm 1.00$ & $73.98 \pm 0.61$ & $97.30 \pm 0.66$ & $87.28 \pm 0.82$ & $89.76 \pm 0.48$   & $98.82 \pm 0.17$   & $88.02 \pm 0.64$    \\
                             & Minmax        & $29.28 \pm 1.12$     & $54.75 \pm 6.47$ & $24.16 \pm 0.55$ & $74.96 \pm 1.00$ & $97.12 \pm 0.67$ & $85.55 \pm 1.25$ & $90.67 \pm 0.16$   & $99.11 \pm 0.12$   & $83.12 \pm 1.35$    \\
                             & Quantile      & $36.88 \pm 1.35$     & $67.59 \pm 0.17$ & $31.14 \pm 0.64$ & $77.24 \pm 0.23$ & $97.48 \pm 0.34$ & $88.44 \pm 0.94$ & $89.19 \pm 1.29$   & $98.86 \pm 0.10$   & $87.29 \pm 0.64$    \\
                             & Robust        & $33.84 \pm 0.62$     & $67.10 \pm 0.45$ & $33.45 \pm 1.15$ & $77.24 \pm 0.23$ & $97.21 \pm 0.34$ & $84.59 \pm 2.88$ & $90.44 \pm 1.11$   & $98.27 \pm 0.10$   & $87.19 \pm 0.00$    \\ \hline
\end{tabular}%
}
\end{table*}

\noindent Consequently, and for both computed inequalities, by taking the Markov Inequality, we get the desired bound, specifically:

\begin{itemize}
    \item For Global-based scaling: $$\gamma = \frac{\epsilon\sqrt{d} }{ \sigma (x^{\max} - x^{\min})}  \left(\frac{d}{d-1}\right)^2 (1+C_1) C_2.$$
    \item For Instance-based scaling: $$\gamma =  \left(\sqrt{\sum_{i=1}^{N} \frac{1}{(x_i^{\max} - x_i^{\min})^2}} \right)\frac{\epsilon}{\sigma} \left(\frac{d}{d-1}\right)^2 (1+C_1) C_2,$$
\end{itemize}

\begin{align*} 
    & \text{with:} \hspace{0.5em}  C_1 = \lVert W_O \lVert \sqrt{H} 
    \max_h \Biggl[ \Bigl\lVert W^{V,h} \Bigr\rVert \\
    &\hspace{7em}\left(4\,\frac{n}{\sqrt{d/H}}\,\Bigl\lVert W^{Q,h} \Bigr\rVert\,\Bigl\lVert W^{K,h} \Bigr\rVert + 1\right) \Biggr], \\
    & C_2 = 1 + \Bigl\lVert W_{FFN} \Bigr\rVert.
\end{align*}

\end{proof}

\section{Additional Results}\label{appendix:additional_results}
To further complete the results that were presented in Section \ref{sec:experimental_results_classification} in which we have chosen to visualize the effect of normalization through Radar chart, we provide the results in form of a table as well. In this perspective, Table \ref{tab:results_classification_appendix} reports the mean (and corresponding standard deviation) accuracy of the considered strategies for the different models and datasets.

\begin{table}[h]
\caption{Statistics of the additional classification datasets used in our experiments.}
\centering
\label{tab:dataset_details}
\resizebox{0.75\columnwidth}{!}{%
\begin{tabular}{lcccc}
\hline
Dataset & \#Training Points & \#Test Points & \#Length & \#Classes \\
\hline
EthanolConcentration  & 261 & 263 & 1751 & 4 \\
FaceDetection         & 5890 & 3524 & 62 & 2 \\
Handwriting           & 150 & 850 & 152 & 26 \\
Heartbeat             & 204 & 205 & 405 & 2 \\
JapaneseVowels        & 270 & 370 & 29 & 9 \\
PEMS-SF               & 267 & 173 & 144 & 7 \\
SelfRegulationSCP1    & 268 & 293 & 896 & 2 \\
SpokenArabicDigits    & 6599 & 2199 & 93 & 10 \\
UWaveGestureLibrary   & 2238 & 2241 & 315 & 8 \\
\hline
\end{tabular}%
}
\end{table}

\section{Experimental Details}\label{appendix:experimental_details}

\noindent \textbf{Implementation.} For all the experiments, we consider fully-supervised learning. The different considered model are optimized using the Adam optimizer (learning rate $10^{-3}$) for 100 epochs, with early stopping (with a patience of $10$ epochs). For the classification task, we optimize the binary cross-entropy loss, while for the regression, we consider the mean squared error loss. All the experiments were run on a L4 Nvidia GPU.

\noindent \textbf{Model Architecture.} For the classification task, we set the hiddden representation of all the models to $128$, and we set the number of encoder layers to $3$. We use the Mean-Pooling to do the final representation that is used before the classification head. For the forecasting, we set the sequence length to $96$, and we use $2$ layers of encoder and $7$ layers for the decoder.

\noindent \textbf{Datasets.} For the Classification dataset, we mainly based our experiment on a set of time series dataset from the UAE database. Additional details about these dataset are provided in Table \ref{tab:dataset_details}.

\end{document}